\documentclass{llncs} 
\usepackage{llncsdoc}
\usepackage{aliascnt}
\usepackage{remreset}
\usepackage{caption}
\usepackage{enumerate}
\usepackage{algorithm}
\usepackage[noend]{algpseudocode}
\usepackage{amsmath}
\usepackage{bm} 
\usepackage{comment}
\usepackage{amssymb}

\usepackage{amsthm}
\usepackage{atbegshi}

\usepackage[pdftex]{graphicx}
\DeclareMathOperator*{\argmin}{argmin}

\newcommand{\diag}{\mathop{\mathrm{diag}}}
\algnewcommand{\parState}[1]{\State%
  \parbox[t]{\dimexpr\linewidth-\algmargin}{\strut #1\strut}}
\graphicspath{ {./images/} }
\title{Scoring and Classifying with Gated Auto-encoders}
\titlerunning{Scoring and Classifying with Gated Auto-encoders}

\author{
    Daniel Jiwoong Im, and Graham W. Taylor}
\institute{
School of Engineering\\
University of Guelph\\
Guelph, On, Canada\\
\email{\{imj,gwtaylor\}@uoguelph.ca}}

\begin{document}

\maketitle

\begin{abstract}
  Auto-encoders are perhaps the best-known non-probabilistic methods
  for representation learning. They are conceptually simple and easy
  to train. Recent theoretical work has shed light on their ability to
  capture manifold structure, and drawn connections to density
  modeling. This has motivated researchers to seek ways of
  auto-encoder scoring, which has furthered their use in
  classification. Gated auto-encoders (GAEs) are an interesting and
  flexible extension of auto-encoders which can learn transformations
  among different images or pixel covariances within images. However,
  they have been much less studied, theoretically or empirically. In
  this work, we apply a dynamical systems view to GAEs, deriving a
  scoring function, and drawing connections to Restricted Boltzmann
  Machines. On a set of deep learning benchmarks, we also demonstrate
  their effectiveness for single and multi-label classification.\looseness=-1
\end{abstract}

\section{Introduction}

Representation learning algorithms are machine learning algorithms
which involve the learning of features or explanatory
factors. Deep learning techniques, which employ several layers of
representation learning, have achieved much recent success in machine
learning benchmarks and competitions, however, most of these successes
have been achieved with purely supervised learning methods and have
relied on large amounts of labeled data
\cite{krizhevsky2012imagenet,szegedy2014going}. Though progress has been slower, it is likely
that unsupervised learning will be important to future advances in
deep learning \cite{bengio2013deep}. 

The most successful and well-known example of non-probabilistic
unsupervised learning is the auto-encoder. Conceptually simple and
easy to train via backpropagation, various regularized variants of the
model have recently been proposed \cite{Rifai2011,Vincent2008,Swersky2011} 
as well as theoretical insights into their operation
\cite{Alain2013,Vincent2010}.

In practice, the latent representation learned by auto-encoders has
typically been used to solve a secondary problem, often
classification. The most common setup is to train a single
auto-encoder on data from all classes and then a classifier is tasked
to discriminate among classes.  However, this contrasts with the way
probabilistic models have typically been used in the past: in that
literature, it is more common to train one model per class and use
Bayes' rule for classification. There are two challenges to
classifying using per-class auto-encoders. First, up until very
recently, it was not known how to obtain the score of data under an
auto-encoder, meaning how much the model ``likes'' an input. Second,
auto-encoders are non-probabilistic, so even if they can be scored,
the scores do not integrate to 1 and therefore the per-class models
need to be calibrated.


Kamyshanska and Memisevic have recently shown
how scores can be computed from an auto-encoder by interpreting it as
a dynamical system \cite{Kamyshanska2013}.  Although the scores do not integrate to 1, they
show how one can combine the unnormalized scores into a generative
classifier by learning class-specific normalizing constants from
labeled data.

In this paper we turn our interest towards a variant of auto-encoders
which are capable of learning higher-order features from data
\cite{Memisevic2011}. The main idea is to learn relations between
pixel intensities rather than the pixel intensities themselves by
structuring the model as a tri-partite graph which connects hidden
units to pairs of images.  If the images are different, the hidden
units learn how the images transform. If the images are the same, the
hidden units encode within-image pixel covariances. Learning such
higher-order features can yield improved results on recognition and
generative tasks.\looseness=-1

We adopt a dynamical systems view of gated auto-encoders,
demonstrating that they can be scored similarly to the classical
auto-encoder. We adopt the framework of \cite{Kamyshanska2013} both
conceptually and formally in developing a theory which yields insights into the
operation of gated auto-encoders. In addition to the theory,
we show in our experiments
that a classification model based on gated auto-encoder scoring can
outperform a number of other representation learning architectures,
including classical auto-encoder scoring. We also demonstrate that
scoring can be useful for the structured output task of multi-label classification.

\section{Gated Auto-encoders}
In this section, we review the gated
auto-encoder (GAE). Due to space constraints, we will not review the
classical auto-encoder. Instead, we direct the reader to the reviews in
\cite{Memisevic2011,Kamyshanska2014} with which we share
notation. Similar to the classical auto-encoder, the GAE
consists 
of an encoder $h(\cdot)$ and decoder $r(\cdot)$. While the standard
auto-encoder processes a datapoint $\mathbf{x}$, the GAE processes input-output
pairs $(\mathbf{x}, \mathbf{y})$.  The GAE is usually trained to
reconstruct $\mathbf{y}$ given $\mathbf{x}$, though it can also be
trained symmetrically, that is, to reconstruct both $\mathbf{y}$ from
$\mathbf{x}$ and $\mathbf{x}$ from
$\mathbf{y}$. Intuitively, the GAE
learns \emph{relations} between the inputs, rather than
representations of the inputs themselves\footnote{Relational features
  can be mixed with standard features by simply adding connections
  that are not gated.}. If $\mathbf{x} \neq
\mathbf{y}$, for example, they represent sequential frames of a video, intuitively,
the mapping units $\mathbf{h}$ learn \emph{transformations}. In the
case that $\mathbf{x} = \mathbf{y}$ (i.e.~the input is copied), the
mapping units learn pixel covariances.

In the simplest form of the GAE, the $M$ hidden (mapping) units are
given by a basis expansion of $\mathbf{x}$ and $\mathbf{y}$. However,
this leads to a parameterization that it is at least quadratic in the
number of inputs and thus, prohibitively large. Therefore, in
practice, $\mathbf{x}$, $\mathbf{y}$, and $\mathbf{h}$ are projected
onto matrices or (``latent factors''), $W^X$, $W^Y$, and $W^H$,
respectively. 
The number of factors, $F$, must be the same for $X$, $Y$, and
$H$. Thus, the model is completely parameterized by
$\theta = \lbrace W^X, W^Y, W^H \rbrace$ such that 
$W^X$ and $W^Y$ are $F \times D$ matrices (assuming both $\mathbf{x}$
and $\mathbf{y}$ are $D$-dimensional) and $W^H$ is an $M \times F$ matrix.
The encoder function is defined by
\begin{equation}
    h(\mathbf{x},\mathbf{y}) = \sigma(W^H((W^X\mathbf{x}) \odot
    (W^Y\mathbf{y}))) \label{eqn:enc}
\end{equation}
where $\odot$ is element-wise multiplication and 
$\sigma(\cdot)$ is an activation function. The decoder function is defined by
\begin{align}
    r(\mathbf{y}|\mathbf{x}, h) = (W^Y)^T( (W^X\mathbf{x}) \odot
    (W^H)^Th(\mathbf{x},\mathbf{y})) .\ \label{eqn:dec1}\\
    r(\mathbf{x}|\mathbf{y}, h) = (W^X)^T( (W^Y\mathbf{y}) \odot
    (W^H)^Th(\mathbf{x},\mathbf{y}))  ,\ \label{eqn:dec2} 
\end{align}
%
Note that the parameters are usually shared between the encoder and decoder.
The choice of whether
to apply a nonlinearity to the output, and the specific form of
objective function will depend on the nature of the inputs, for
example, binary, categorical, or real-valued. Here, we have assumed
real-valued inputs for simplicity of presentation, therefore,
Eqs.~\ref{eqn:dec1} and \ref{eqn:dec2} are bi-linear functions of $\mathbf{h}$
and we use a squared-error objective:
\begin{equation}
    J = \frac{1}{2} \|r(\mathbf{y}|\mathbf{x}) -\mathbf{y}\|^2.
\end{equation}
We can also constrain the GAE to be a symmetric model by 
 training it to reconstruct both $\mathbf{x}$ given $\mathbf{y}$ and 
$\mathbf{y}$ given $\mathbf{x}$ \cite{Memisevic2011}:
\begin{equation}
    J = \frac{1}{2}  \|r(\mathbf{y}|\mathbf{x}) -\mathbf{y}\|^2 +
    \frac{1}{2} \|r(\mathbf{x}|\mathbf{y}) -\mathbf{x} \|^2 .\
    \label{symobj}
\end{equation}

The symmetric objective can be thought of as the non-probabilistic
analogue of modeling a \emph{joint} distribution over $\mathbf{x}$ and
$\mathbf{y}$ as opposed to a conditional \cite{Memisevic2011}.

\section{Gated Auto-Encoder Scoring} 

In \cite{Kamyshanska2013}, the authors showed that data could be
scored under an auto-encoder by interpreting the model as a
\emph{dynamical system}. In contrast to the probabilistic views based
on score matching \cite{Swersky2011,Vincent2010,Alain2013} and
regularization, the dynamical systems approach permits scoring under
models with either linear (real-valued data) or sigmoid (binary data)
outputs, as well as arbitrary hidden unit activation functions. The
method is also agnostic to the learning procedure used to train the
model, meaning that it is suitable for the various types of
regularized auto-encoders which have been proposed recently. In this
section, we demonstrate how the dynamical systems view can be extended
to the GAE.

\subsection{Vector field representation} \label{vfr}
Similar to \cite{Kamyshanska2013}, we will view the GAE as a dynamical system
with the vector field defined by
\begin{align*}
    F(\mathbf{y}|\mathbf{x}) = r(\mathbf{y}|\mathbf{x}) -\mathbf{y} .\
\end{align*}
The vector field represents the local transformation that
$\mathbf{y}|\mathbf{x}$ undergoes as a result of applying the
reconstruction function $r(\mathbf{y}|\mathbf{x})$. Repeatedly
applying the reconstruction function to an input
$\mathbf{y}|\mathbf{x} \rightarrow r(\mathbf{y}|\mathbf{x})
\rightarrow r(r(\mathbf{y}|\mathbf{x})|\mathbf{x}) \rightarrow \cdots \rightarrow$
$r(r\cdots r(\mathbf{y}|\mathbf{x})|\mathbf{x})$ yields a trajectory whose dynamics, from
a physics perspective, can be
viewed as a force field.
At any point, the potential force acting on a point is the gradient
of some potential energy (negative goodness) at that point. In this
light, the GAE reconstruction may be viewed as pushing pairs of inputs
$\mathbf{x,y}$ in the direction of lower energy.


Our goal is to derive the energy function, which we call a scoring
function, and which measures how much a GAE ``likes'' a given
pair of inputs $(\mathbf{x},\mathbf{y})$ up to normalizing constant.
In order to find an expression for the potential energy, the vector
field must be able to be written as the derivative of a scalar field
\cite{Kamyshanska2013}.  To check this, we can submit to Poincar\'e's
integrability criterion: For some open, simple connected set~~$\mathcal{U}$, a continuously differentiable function $F:\mathcal{U}
\rightarrow \Re^m$ defines a gradient field if and only if
\begin{align}
    \frac{\partial F_i(\mathbf{y})}{\partial y_j} = 
    \frac{\partial F_j(\mathbf{y})}{\partial y_i}, \text{ } \forall i,j=1 \cdots n.\nonumber
\end{align}
The vector field defined
by the GAE indeed satisfies Poincar\'e's integrability criterion;
therefore it can be written as the derivative of a scalar field. A
derivation is given in the Appendix~\ref{App:AppendixA.1}
This also applies to the GAE with a symmetric objective function (Eq.~\ref{symobj})
by setting the input as $\boldsymbol{\xi}|\boldsymbol{\gamma}$ such that 
$\boldsymbol{\xi}=[\mathbf{y};\mathbf{x}]$ and $\boldsymbol{\gamma}=[\mathbf{x};\mathbf{y}]$ and
following the exact same procedure.

\subsection{Scoring the GAE}

As mentioned in Section \ref{vfr}, our goal is to find an energy
surface, so that we can express the energy for a specific pair
$(\mathbf{x}, \mathbf{y})$. From the previous section, we showed that
Poincar\'e's criterion is satisfied and this implies that we can write
the vector field as the derivative of a scalar field. Moreover, it
illustrates that this vector field is a conservative field and this
means that the vector field is a gradient of some scalar function,
which in this case is the energy function of a GAE:
\begin{align*}
    r(\mathbf{y}|\mathbf{x}) - \mathbf{y} = \nabla E\nonumber .\
\end{align*}
Hence, by integrating out the trajectory of the GAE $(\mathbf{x},\mathbf{y})$,
we can measure the energy along a path. Moreover, the line integral
of a conservative vector field is path independent, which 
allows us to take the anti-derivative of the scalar function:
\begin{align}
    E(\mathbf{y}|\mathbf{x}) =& \int (r(\mathbf{y}|\mathbf{x})- \mathbf{y}) d\mathbf{y} = \int W^Y\left( \left(W^X\mathbf{x}) \odot W^Hh(\mathbf{u}\right)\right) d\mathbf{y} - \int  \mathbf{y} d\mathbf{y}\nonumber\\
    =& W^Y \left( \left(W^X\mathbf{x}\right) \odot  W^H\int h\left(\mathbf{u}\right) d\mathbf{y}\right) - \int \mathbf{y} d\mathbf{y}
    \label{exgy},
\end{align}
where $\mathbf{u}$ is an auxiliary variable such that $\mathbf{u} = W^H((W^Y\mathbf{y}) \odot (W^X\mathbf{x})) $ and 
    $\frac{d\mathbf{u}}{d\mathbf{y}} = W^H(W^Y \odot (W^X\mathbf{x} \otimes \mathbf{1}_D ))$, 
and $\otimes$ is the Kronecker product. Moreover, the decoder can be re-formulated as
\begin{align*}
    r(\boldsymbol{y}|\boldsymbol{x}) &= (W^Y)^T( W^X\boldsymbol{x} \odot (W^H)^Th(\boldsymbol{y},\boldsymbol{x})) \\
                                     &= \left((W^Y)^T \odot
                                       (W^X\boldsymbol{x} \otimes
                                       \boldsymbol{1}_D)\right)
                                     (W^H)^Th(\boldsymbol{y},\boldsymbol{x}) .\
\end{align*}
Re-writing Eq.~\ref{exgy} in terms of the auxiliary variable $\mathbf{u}$, we get
\begin{align}
    E(\mathbf{y}|\mathbf{x}) =& \left((W^Y)^T \odot  (W^Y\mathbf{x} \otimes \mathbf{1}_D)\right) (W^H)^T \\&\int h(\mathbf{u}) 
                              \left(W^H\left( W^Y \odot (W^X\mathbf{x}
                                  \otimes \mathbf{1}_D
                          )\right)\right)^{-1}d\mathbf{u} - \int
                                  \mathbf{y} d\mathbf{y} \nonumber \\
    = & \int h(\mathbf{u}) d\mathbf{u} - \frac{1}{2}\mathbf{y}^2 +
    \text{const} .\
    \label{exgy2}
\end{align}
A more detailed derivation from Eq.~\ref{exgy} to Eq.~\ref{exgy2} is
provided in the Appendix~\ref{App:AppendixA.1}.
%
 Identical to \cite{Kamyshanska2013}, if $h(\mathbf{u})$ is an element-wise activation function and
 we know its anti-derivative, then it is very 
simple to compute $E(\mathbf{x},\mathbf{y})$.

\section{Relationship to  Restricted Boltzmann Machines}
In this section, we relate GAEs through the scoring
function to other types of Restricted Boltzmann Machines, such
as the Factored Gated Conditional RBM \cite{Taylor2009} and the
Mean-covariance RBM \cite{Ranzato2010}.

\subsection{Gated Auto-encoder and Factored Gated Conditional Restricted Boltzmann Machines}
Kamyshanska and Memisevic showed that several hidden activation
functions defined gradient fields, including sigmoid, softmax, tanh,
linear, rectified linear function (ReLU), modulus, and squaring. These
activation functions are applicable to GAEs as well.

In the case of the sigmoid activation function, $\sigma = h(\mathbf{u}) = \frac{1}{1+\exp{(-\mathbf{u})}}$,
our energy function becomes
\begin{align}
    E_{\sigma} = &2\int (1+\exp{-(\mathbf{u})})^{-1} d\mathbf{u} -
    \frac{1}{2}(\mathbf{x}^2 + \mathbf{y}^2)+ \text{const} ,\ \nonumber \\
        = &2\sum_k \log{( 1+\exp{(W^H_{k\cdot}(W^X\mathbf{x}\odot W^X\mathbf{y}))})} 
     - \frac{1}{2}(\mathbf{x}^2 + \mathbf{y}^2)+ \text{const} .\ \nonumber 
\end{align}
Note that if we consider the conditional GAE we reconstruct $\mathbf{x}$
given $\mathbf{y}$ only, this yields
\begin{equation}
    E_\sigma(\mathbf{y}|\mathbf{x}) = \sum_k \log{( 1+\exp{(W^H(W^Y_{k\cdot}\mathbf{y}\odot W^X_{k\cdot}\mathbf{x}))})}
        - \frac{\mathbf{y}^2}{2}+ \text{const} .\
\end{equation}
This expression is identical, up to a constant, to the free energy in
a Factored Gated Conditional Restricted Boltzmann Machine (FCRBM) with
Gaussian visible units and Bernoulli hidden units. We have ignored
biases for simplicity.
A derivation including biases is provided in the Appendix~\ref{App:AppendixB.1}.

\subsection{Mean-Covariance Auto-encoder and Mean-covariance Restricted Boltzmann Machines}
The Covariance auto-encoder (cAE) was introduced in
\cite{Memisevic2011}. It is a specific form of symmetrically trained
auto-encoder with identical inputs: $\mathbf{x} = \mathbf{y}$, and
tied input weights: $W^X = W^Y$. It
maintains a set of relational mapping units to model
covariance between pixels. One can introduce a separate set of mapping
units connected pairwise to only one of the inputs which model the
mean intensity. In this case, the model becomes a Mean-covariance
auto-encoder (mcAE).
\begin{theorem}
    Consider a cAE with encoder and decoder:
    \begin{align*}
        h(\mathbf{x}) &= h(W^H((W^X\mathbf{x})^2) + \mathbf{b})\nonumber \\
        r(\mathbf{x}| h) &= (W^X)^T(W^X\mathbf{x} \odot (W^H)^Th(\mathbf{x}))  + \mathbf{a},\nonumber 
    \end{align*}
     \noindent where $\theta = \lbrace W^X, W^H,
    \mathbf{a}, \mathbf{b}\rbrace$ are the parameters of the model,
    and $h(\mathbf{z})=\frac{1}{1+\exp{(-\mathbf{z}})}$ is a sigmoid.
    Moreover, consider a Covariance RBM \cite{Ranzato2010} with
    Gaussian-distributed visibles and Bernoulli-distributed hiddens, 
    with an energy function defined by
    \begin{equation*}
        E^c(\mathbf{x},\mathbf{h}) = \frac{(\mathbf{a}-\mathbf{x})^2}{\sigma^2}- \sum_f P\mathbf{h} (C\mathbf{x})^2 -\mathbf{bh}.
    \end{equation*}
    Then the energy function of the cAE with dynamics 
    $r(\mathbf{x}|\mathbf{y}) - \mathbf{x}$ is equivalent to the free
    energy of Covariance RBM
    up to a constant:
    \begin{align}
        E(\mathbf{x},\mathbf{x}) = \sum_k \log
        \left(1+\exp\left(W^H(W^X\mathbf{x})^2+\mathbf{b}\right)\right) -
        \frac{\mathbf{x}^2}{2} + \textnormal{const} .\
    \end{align}

    \label{covtheorem}
\end{theorem}
The proof is given in the Appendix~\ref{App:AppendixB.2}. 
We can extend this analysis to the mcAE
by using the above theorem and the results from \cite{Kamyshanska2013}.
\begin{corollary}
    The energy function of a mcAE and the free
    energy of a Mean-covariance RBM
    (mcRBM) with Gaussian-distributed visibles and
    Bernoulli-distributed hiddens are equivalent up to a constant. The
    energy of the mcAE is:
    \begin{equation}
        E = \sum_k \log \left(1+\exp\left(-W^H(W^X\mathbf{x})^2-\mathbf{b}\right) \right)
        + \sum_k \log \left(1+\exp (W\mathbf{x}+\mathbf{c})\right)- \mathbf{x}^2+\textnormal{const}
    \end{equation}
    where $\theta^m = \lbrace W, \mathbf{c}\rbrace$ parameterizes the
    mean mapping units and $\theta^c = \lbrace W^X, W^H,$ $\mathbf{a},
    \mathbf{b}\rbrace$ parameterizes the covariance mapping units.
\end{corollary}
\begin{proof}
    The proof is very simple. Let $E_{mc} = E_m + E_c$, where $E_m$ is
    the energy of the mean auto-encoder, $E_c$ is
    the energy of the covariance auto-encoder, and $E_{mc}$ is the
    energy of the mcAE. We know
    from Theorem \ref{covtheorem} that $E_c$ is equivalent to the free
    energy of a covariance RBM, and the results
    from \cite{Kamyshanska2013} show that that $E_m$ is equivalent to
    the free energy of mean (classical) RBM. As shown in
    \cite{Ranzato2010}, the free energy of a mcRBM is equal to
    summing the free energies of a mean RBM and a
    covariance RBM.
%
\end{proof}

\section{Classification with Gated Auto-encoders}

Kamyshanska and Memisevic demonstrated that one application of the
ability to assign energy or scores to auto-encoders was in constructing a
classifier from class-specific auto-encoders. In this section, we
explore two different paradigms for classification. Similar to
that work, we consider the usual multi-class problem by
first training class-specific auto-encoders, and using their energy
functions as confidence scores. We also consider the more challenging
structured output problem, specifically, the case of multi-label
prediction where a data point may have more than one associated label,
and there may be correlations among the labels.

\subsection{Classification using class-specific gated auto-encoders}

One approach to classification is to take several class-specific
models and assemble them into a classifier. The best-known example of this approach
is to fit several directed graphical models and use Bayes' rule to
combine them. The process is simple because the models are normalized, or
calibrated. While it is possible to apply a similar technique to undirected
or non-normalized models such as auto-encoders, one must take
care to calibrate them. 

The approach proposed in \cite{Kamyshanska2013} is to train $K$
class-specific auto-encoders, each of which assigns a non-normalized
energy to the data $E_i \left(\mathbf{x} \right), i=1 \ldots,K$, and
then define the conditional distribution over classes $z_i$ as
\begin{align}
P(z_i|\mathbf{x}) = \frac{\exp \left( E_i \left( \mathbf{x} \right) +
    B_i \right)}{\sum_j \exp \left( E_j \left( \mathbf{x} \right) +
    B_j \right)} ,\ \label{eqn:gae:classification}
\end{align}
\noindent where $B_i$ is a learned bias for class $i$. The bias terms
take the role of calibrating the unnormalized energies. Note that we
can similarly combine the energies from a symmetric gated auto-encoder where $\mathbf{x}=\mathbf{y}$ (i.e.~a
covariance auto-encoder) and apply
Eq.~\ref{eqn:gae:classification}. If, for each class, we train both a
covariance auto-encoder and a classical auto-encoder (i.e.~a ``mean''
auto-encoder) then we can combine both sets of unnormalized energies
as follows
\begin{equation}
    P_{mcAE}(z_i|\mathbf{x}) = \frac{\exp(E^M_i(\mathbf{x}) + E^C_i(\mathbf{x})+B_i)}
                        {\sum_j\exp(E^M_j(\mathbf{x})+
                          E^C_j(\mathbf{x})+B_j)} ,\ \label{GAES_eq}
\end{equation}

\noindent where $E_i^M(\mathbf{x})$ is the energy which comes from the
``mean'' (standard) auto-encoder trained on class $i$ and
$E_i^C(\mathbf{x})$ the energy which comes from the ``covariance''
(gated) auto-encoder trained on class $i$. We call the classifiers in
Eq.~\ref{eqn:gae:classification} and Eq.~\ref{GAES_eq} ``Covariance
Auto-encoder Scoring'' (cAES) and ``Mean-Covariance Auto-encoder
Scoring'' (mcAES),
respectively.

The training procedure is summarized as follows:
\begin{enumerate}
    \item Train a (mean)-covariance
      auto-encoder individually 
      for each class. Both the mean and covariance auto-encoder have
      tied weights in the encoder and decoder. The covariance
      auto-encoder is a gated auto-encoder with tied inputs. 
    \item Learn the $B_i$ calibration terms using maximum
      likelihood, and backpropagate to the GAE parameters.

\end{enumerate}

\subsubsection{Experimental results}

We followed the same experimental setup as \cite{Memisevic2010}
where we used a standard set of ``Deep Learning Benchmarks'' \cite{Larochelle2007}.
We used mini-batch stochastic gradient descent to optimize parameters
during training. The hyper-parameters: number of hiddens,
number of factors, corruption level, learning rate, weight-decay, momentum
rate, and batch sizes were chosen based on a held-out validation set. Corruption
levels and weight-decay were selected from $\lbrace 0, 0.1,0.2,0.3,0.4,0.5\rbrace$,
and number of hidden and factors were selected from 100,300,500. We
selected the learning rate and weight-decay from the range $(0.001, 0.0001)$.

Classification error results are shown in Table \ref{mcaes_err}. 
First, the error rates of auto-encoder scoring variant methods
illustrate that across all datasets
AES outperforms cAES and mcAES outperforms both AES and cAES.
AE models pixel means and cAE models pixel covariance, while mcAE models both mean and 
 covariance, making it naturally more expressive.
We observe that cAES and mcAES achieve lower error rates by a large margin on rotated MNIST
with backgrounds (final row). On the other hand, both cAES and mcAES perform poorly on
MNIST with random white noise background (second row from bottom). We
believe this phenomenon is due to the inability to model covariance in
this dataset. In MNIST with random white noise the pixels are
typically uncorrelated, where in rotated MNIST with backgrounds the
correlations are present and consistent.

\begin{table}[htpt]
\begin{center}
\begin{small}
\begin{tabular}{lccccccc}
\hline
DATA & SVM & RBM & DEEP & GSM & AES & cAES & mcAES\\
     &{\tiny RBF} &     & {\tiny SAA$_3$}    &     &     \\    
\hline
RECT            & 2.15  & 4.71 & 2.14       & 0.56      & 0.84       & 0.61 & \bf{0.54}\\
RECT\tiny{IMG}  & 24.04 & 23.69& 24.05      & 22.51     & 21.45      & 22.85& \bf{21.41} \\
CONVEX          & 19.13 & 19.92& 18.41      & \bf{17.08}& 21.52      & 21.6 & 20.63 \\
MNIST\tiny{SMALL}&3.03  & 3.94 & 3.46       & 3.70      & \bf{2.61}  & 3.65 & 3.65  \\
MNIST\tiny{ROT} & 11.11 & 14.69& \bf{10.30} & 11.75     & 11.25      & 16.5 & 13.42 \\
MNIST\tiny{RAND}& 14.58 & 9.80 & 11.28      & 10.48     & \bf{9.70}  & 18.65& 16.73 \\
MNIST\tiny{ROTIM}&55.18 & 52.21& 51.93      & 55.16     & 47.14      & 39.98& \bf{35.52} \\
\hline
\end{tabular}
\caption{Classification error rates on the Deep Learning Benchmark dataset. SAA$_3$ stands for three-layer
    Stacked Auto-encoder. SVM and RBM results are from \cite{Vincent2010}, DEEP and GSM are
    results from \cite{Memisevic2011}, and AES is from \cite{Kamyshanska2013}.}
\label{mcaes_err}
\end{small}
\end{center}
\vskip -0.1in
\end{table}

\subsection{Multi-label classification via optimization in label
  space} \label{str_method}

The dominant application of deep learning approaches to vision has
been the assignment of images to discrete classes (e.g.~object
recognition).  Many applications, however, involve ``structured
outputs'' where the output variable is high-dimensional and has a
complex, multi-modal joint distribution. Structured output prediction
may include tasks such as multi-label classification where there are
regularities to be learned in the output, and segmentation, where the
output is as high-dimensional as the input. A key challenge to such
approaches lies in developing models that are able to capture complex,
high level structure like shape, while still remaining tractable.

Though our proposed work is based on a deterministic model, we have
shown that the energy, or scoring function of the GAE is equivalent,
up to a constant, to that of a conditional RBM, a model that has
already seen some use in structured prediction problems
\cite{Mnih2011,Li2013}.

GAE scoring can be applied to structured output problems as a type of
``post-classification'' \cite{MnihHinton2010}. The idea is to let a
nai\"ve, non-structured classifier make an initial prediction of the
outputs in a fast, feed-forward manner, and then allow a second model
(in our case, a GAE) clean up the outputs of the first model. Since
GAEs can model the relationship between input $\mathbf{x}$ and
structured output $\mathbf{y}$, we can initialize the output with the
output of the nai\"ve model, and then optimize its energy function with
respect to the outputs. Input $\mathbf{x}$ is held constant throughout
the optimization.
%
%

Li \textit{et al} recently proposed Compositional High Order Pattern
Potentials, a hybrid of Conditional Random
Fields (CRF) and Restricted Boltzmann Machines. The RBM provides a
global shape information prior to the locally-connected CRF. Adopting
the idea of \emph{learning} structured relationships between
outputs, we propose an alternate approach which the inputs of the GAE
are not $(\mathbf{x},\mathbf{y})$ but $(\mathbf{y},\mathbf{y})$. In
other words, the post-classification model is a covariance
auto-encoder. The intuition behind the first approach is to use a GAE
to learn the relationship between the input $\mathbf{x}$ and the
output $\mathbf{y}$, whereas the second method aims to learn the
correlations between the outputs $\mathbf{y}$.  

We denote our two proposed methods GAE$_{XY}$ and
GAE$_{Y^2}$. GAE$_{XY}$ corresponds to a GAE, trained conditionally,
whose mapping units directly model the relationship between input and
output and GAE$_{Y^2}$ corresponds to a GAE which models correlations
between output dimensions. GAE$_{XY}$ defines $E
\left(\mathbf{y}|\mathbf{x}\right)$, while GAE$_{Y^2}$ defines $E
\left(\mathbf{y}|\mathbf{y}\right) = E(\mathbf{y})$. 
They differ only in terms of the data vectors that they consume. The
training and test procedures are detailed in Algorithm
\ref{algo:SOLRG}.


%

\begin{algorithm}[h]
    \caption{\small{Structured Output Prediction with GAE scoring}}
    \label{algo:SOLRG}
    {\small 
    \begin{algorithmic}[1]
        \Procedure{Multi-label Classification}{$\mathcal{D} = \lbrace (\mathbf{x}_i,\mathbf{y}_i) \in \mathcal{X}_{train} \times \mathcal{Y}_{train} \rbrace$ }
        \State{Train a Multi-layer Perceptron (MLP) to learn an input-output mapping $f(\cdot)$:
            \begin{align}
                \argmin_{\theta_1} l(\mathbf{x},\mathbf{y};\theta_1) = \sum_i
                \text{loss}_1 \left( (f \left(\mathbf{x}_{i}; \theta_1
                  \right) -
                \mathbf{y}_{i} \right)
            \end{align}
        \indent where $\text{loss}_1$ is an appropriate loss function for the MLP.\footnotemark}
    \State{Train a Gated Auto-encoder with inputs $(\mathbf{x}_i,
        \mathbf{y}_{i});$ For the case of GAE$_{Y^2}$, set $\mathbf{x}_i=\mathbf{y}_i$. 
        \begin{align}
            \argmin_{\theta_2} l(\mathbf{x},\mathbf{y};\theta_2) = \sum_i
            \text{loss}_2 \left(
              r(\mathbf{y}_{i}|\mathbf{x}_{i}, \theta_2) -
              \mathbf{y}_{i} \right)
        \end{align}
        \indent where $\text{loss}_2$ is an appropriate reconstructive loss for the auto-encoder.}
        \For {each test data point $\mathbf{x_i} \in \mathcal{X}_{test}$}
        \State{Initialize the output using the MLP. 
            \begin{align}
                \mathbf{y}_0 = f \left(\mathbf{x}_{test} \right)
            \end{align}
            \While {$\|E(\mathbf{y}_{t+1}|\mathbf{x})- E(\mathbf{y}_{t}|\mathbf{x})\| > \epsilon$ or $\leq$ max. iter.}
                \State{Compute $\triangledown_{\mathbf{y}_{t}} E$}
                \State{Update $\mathbf{y}_{t+1} = \mathbf{y}_t - \lambda \triangledown_{\mathbf{y}_t} E$}
                \State{where $\epsilon$ is the tolerance rate with respect to the
                convergence of the optimization.}
            \EndWhile 
        } 
        \EndFor
        \EndProcedure
    \end{algorithmic}}
\end{algorithm}
\AtBeginShipoutNext{\footnotetext{In our experiments, we used the cross-entropy loss function for loss$_1$ and loss$_2$.}}

\subsubsection{Experimental results}

We consider multi-label classification, where
the problem is to classify instances which can take on more than one
label at a time. We followed
the same experimental set up as \cite{Mnih2011}. Four multi-labeled
datasets were considered:
Yeast \cite{Elisseff2002} consists of biological attributes, Scene \cite{Boutell2004}
is image-based, and MTurk \cite{Mandel2010} and MajMin
\cite{Mandel2008} are targeted towards tagging music. Yeast consists of 103 biological
attributes and has 14 possible labels, Scene consists of 294 image pixels with
6 possible labels, and MTurk and MajMin each consist of 389 audio features extracted from
 music and have 92 and 96 possible tags, respectively. Figure \ref{fig:data_cov} visualizes the covariance matrix
for the label dimensions in each dataset. We can see from this that
there are correlations present in the labels which suggests that a
structured approach may improve on a non-structured predictor.

\begin{figure}[htp]
    \centering
    \includegraphics[width=1.0\textwidth]{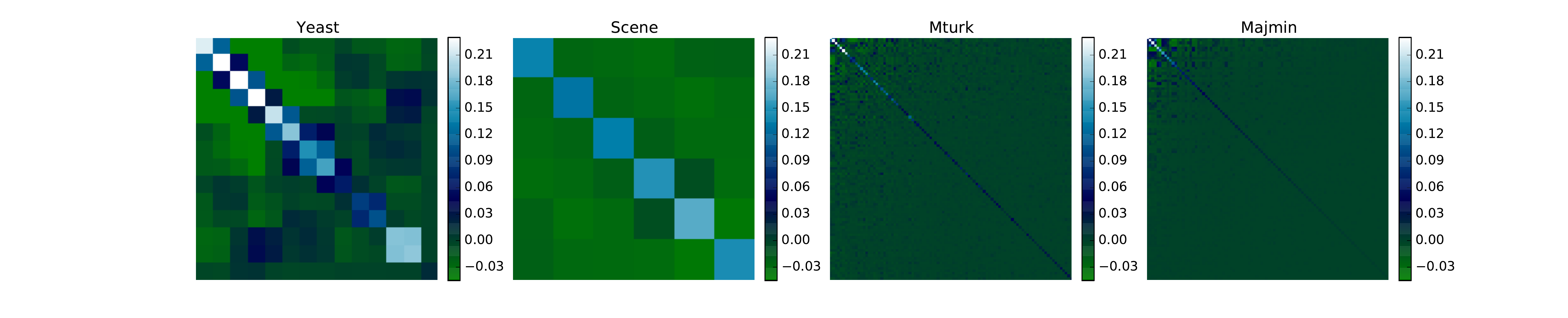}
\caption{Covariance matrices for the multi-label datasets: Yeast,
  Scene, MTurk, and MajMin.}
\label{fig:data_cov}
    \vspace{-0.4cm}
\end{figure}

We compared our proposed approaches to logistic regression, a standard
MLP, and the two structured CRBM training
algorithms presented in \cite{Mnih2011}.
To permit a fair comparison, we followed the same procedure for training and
reporting errors as in that paper, where we cross validated over 10
folds and training, validation, test examples
are randomly separated into 80\%, 10\%, and 10\% in each fold. 
The error rate was measured by averaging the errors on each label dimension.

\begin{table}
    \centering
    \begin{small}
        \begin{tabular}{lcccr}
        \hline
        Method & Yeast & Scene & MTurk & MajMin\\
        \hline
        LogReg       & 20.16  & 10.11  & 8.10 & 4.34\\
        HashCRBM$^*$ & 20.02  & 8.80   & 7.24 & 4.24\\
        MLP          & 19.79  & 8.99   & 7.13  & 4.23 \\
        GAES$_{XY}$  & \textbf{19.27}  & 6.83   & \textbf{6.59}  & \textbf{3.96} \\
        GAES$_{Y^2}$ & 19.58  & \textbf{6.81}   & \textbf{6.59}  & 4.29 \\
        \end{tabular}

        \caption{Error rate on multi-label datasets. As in previous
          work, we report the mean across 10 repeated runs with
          different random weight initializations.}
        \label{tab:performance}
    \end{small}
\end{table}

The performance on four multi-label datasets is shown in Table \ref{tab:performance}.
We observed that adding a small amount of Gaussian noise to the input $\mathbf{y}$ improved the performance for GAE$_{XY}$. However,
adding noise to the input $\mathbf{x}$ did not have as much of an effect. 
We suspect that adding noise makes the GAE more robust to the input provided by the MLP.
 Interestingly, we found that the performance of GAE$_{Y^2}$ was negatively affected
 by adding noise. Both of our proposed methods, GAES$_{XY}$ and GAES$_{Y^2}$
 generally outperformed the other methods except for GAES$_{Y^2}$ on the MajMin dataset.
At least for these datasets, there is no clear winner between the two.
 GAES$_{XY}$ achieved lower error than GAES$_{Y^2}$ for Yeast and MajMin,
 and the same error rate on the MTurk dataset. However, GAES$_{Y^2}$ outperforms GAES$_{XY}$ on the Scene dataset.
 Overall, the results show that GAE scoring may be a promising means of post-classification in structured output prediction.

\section{Conclusion}
There have been many theoretical and empirical studies on
auto-encoders \cite{Vincent2008,Rifai2011,Swersky2011,Vincent2010,Alain2013,Kamyshanska2013},
however, the theoretical study of gated
auto-encoders is limited apart from \cite{Memisevic2011,Alain2013b}.
The GAE has several intriguing properties that a classical
auto-encoder does not, based on its ability to model relations among
pixel intensities rather than just the intensities themselves. This opens
up a broader set of applications.  In this paper, we
derive some theoretical results for the GAE that enable us to gain
more insight and understanding of its operation.

We cast the GAE as a dynamical system driven by a vector field in
order to analyze the model. In the first part of the paper, by
following the same procedure as \cite{Kamyshanska2013}, we showed that
the GAE could be scored according to an energy function. From this
perspective, we demonstrated the equivalency of the GAE energy to the
free energy of a FCRBM with Gaussian visible units, Bernoulli hidden
units, and sigmoid hidden activations.  In the same manner, we also
showed that the covariance auto-encoder can be formulated in a way
such that its energy function is the same as the free energy of a
covariance RBM, and this naturally led to a connection between the mean-covariance
auto-encoder and mean-covariance RBM.
One interesting observation is
that Gaussian-Bernoulli RBMs have been reported to be difficult to
train \cite{Krizhevsky2009,Cho2011}, and the success of training RBMs
is highly dependent on the training setup
\cite{Wang2014}. Auto-encoders are an attractive alternative, even when an energy function is required.

Structured output prediction is a natural next step for representation learning.
The main advantage of our approach compared to other popular approaches such
as Markov Random Fields, is that inference is extremely fast, using a gradient-based 
optimization of the auto-encoder scoring function. In the future, we plan on 
tackling more challenging structured output prediction problems.

\bibliography{scoring_and_classification_with_gae}

\begin{thebibliography}{10}
\providecommand{\url}[1]{\texttt{#1}}
\providecommand{\urlprefix}{URL }

\bibitem{bengio2013deep}
Bengio, Y., Thibodeau-Laufer, {\'E}.: Deep generative stochastic networks
  trainable by backprop. arXiv preprint arXiv:1306.1091  (2013)

\bibitem{Boutell2004}
Boutell, M.R., Luob, J., Shen, X., Brown, C.M.: Learning multi-label scene
  classification. Pattern Recognition  37,  1757--1771 (2004)

\bibitem{Cho2011}
Cho, K., Ilin, A., Raiko, T.: Improved learning of {Gaussian-Bernoulli}
  restricted {Boltzmann} machines. In: ICANN. pp. 10--17 (2011)

\bibitem{Alain2013b}
Droniou, A., Sigaud, O.: Gated autoencoders with tied input weights. In: ICML
  (2013)

\bibitem{Elisseff2002}
Elisseeff, A., Weston, J.: A kernel method for multi-labelled classification.
  In: NIPS (2002)

\bibitem{Alain2013}
Guillaume, A., Bengio, Y.: What regularized auto-encoders learn from the data
  generating distribution. In: ICLR (2013)

\bibitem{Kamyshanska2013}
Kamyshanska, H., Memisevic, R.: On autoencoder scoring. In: ICML. pp. 720--728
  (2013)

\bibitem{Kamyshanska2014}
Kamyshanska, H., Memisevic, R.: The potential energy of an auto-encoder. IEEE
  Transactions on Pattern Analysis and Machine Intelligence  37(6),  1261--1273
  (2014)

\bibitem{Krizhevsky2009}
Krizhevsky, A.: Learning multiple layers of features from tiny images. Tech.
  rep., Department of Computer Science, University of Toronto (2009)

\bibitem{krizhevsky2012imagenet}
Krizhevsky, A., Sutskever, I., Hinton, G.E.: Imagenet classification with deep
  convolutional neural networks. In: NIPS (2012)

\bibitem{Larochelle2007}
Larochelle, H., Erhan, D., Courville, A., Bergstra, J., Bengio, Y.: An
  empirical evaluation of deep architectures on problems with many factors of
  variation. In: ICML (2007)

\bibitem{Li2013}
Li, Y., Tarlow, D., Zemel, R.: Exploring compositional high order pattern
  potentials for structured output learning. In: CVPR (2013)

\bibitem{Mandel2010}
Mandel, M.I., Eck, D., Bengio, Y.: Learning tags that vary within a song. In:
  ISMIR (2010)

\bibitem{Mandel2008}
Mandel, M.I., Ellis, D.P.W.: A web-based game for collecting music metadata.
  Journal New of Music Research  37,  151--165 (2008)

\bibitem{Memisevic2011}
Memisevic, R.: Gradient-based learning of higher-order image features. In: ICCV
  (2011)

\bibitem{Memisevic2010}
Memisevic, R., Zach, C., Hinton, G., Pollefeys, M.: Gated softmax
  classification. In: NIPS (2010)

\bibitem{MnihHinton2010}
Mnih, V., Hinton, G.: Learning to detect roads in high-resolution aerial
  images. In: Proceedings of the 11th European Conference on Computer Vision
  (ECCV) (2010)

\bibitem{Mnih2011}
Mnih, V., Larochelle, H., Hinton, G.E.: Conditional restricted {Boltzmann}
  machines for structured output prediction. In: UAI (2011)

\bibitem{Ranzato2010}
Ranzato, M., Hinton, G.E.: Modeling pixel means and covariances using
  factorized third-order {Boltzmann} machines. In: CVPR (2010)

\bibitem{Rifai2011}
Rifai, S.: Contractive auto-encoders: Explicit invariance during feature
  extraction. In: ICML (2011)

\bibitem{Swersky2011}
Swersky, K., Ranzato, M., Buchman, D., Freitas, N.D., Marlin, B.M.: On
  autoencoders and score matching for energy based models. In: ICML. pp.
  1201--1208 (2011)

\bibitem{szegedy2014going}
Szegedy, C., Liu, W., Jia, Y., Sermanet, P., Reed, S., Anguelov, D., Erhan, D.,
  Vanhoucke, V., Rabinovich, A.: Going deeper with convolutions. arXiv preprint
  arXiv:1409.4842  (2014)

\bibitem{Taylor2009}
Taylor, G.W., Hinton, G.E.: Factored conditional restricted {Boltzmann}
  machines for modeling motion style. In: ICML. pp. 1025--1032 (2009)

\bibitem{Vincent2010}
Vincent, P.: A connection between score matching and denoising auto-encoders.
  Neural Computation  23(7),  1661--1674 (2010)

\bibitem{Vincent2008}
Vincent, P., Larochelle, H., Bengio, Y., Manzagol, P.: Extracting and composing
  robust features with denoising autoencoders. In: ICML (2008)

\bibitem{Wang2014}
Wang, N., Melchior, J., Wiskott, L.: {Gaussian}-binary restricted {Boltzmann}
  machines on modeling natural image statistics. Tech. rep., Institut fur
  Neuroinformatik Ruhr-Universitat Bochum, Bochum, 44780, Germany (2014)

\end{thebibliography}
\bibliographystyle{splncs03}

\ifx
\newpage
\appendix

\section{Gated Auto-encoder Scoring} \label{App:AppendixA}

\subsection{Vector field representation} \label{App:poincare}
To check that the vector
field can be written as the derivative of a scalar field,
we can submit to Poincar\'e's
integrability criterion: For some open, simple connected set
$\mathcal{U}$, a continuously differentiable function $F:\mathcal{U}
\rightarrow \Re^m$ defines a gradient field if and only if
\begin{align}
    \frac{\partial F_i(\mathbf{y})}{\partial y_j} = 
    \frac{\partial F_j(\mathbf{y})}{\partial y_i}, \text{ } \forall i,j=1 \cdots n.\nonumber
\end{align}
Considering the GAE, note that $i^{th}$ component of the decoder
$r_i(\mathbf{y}|\mathbf{x})$ can be rewritten as
\begin{align*}
    r_i(\mathbf{y}|\mathbf{x}) &= (W^Y_{\cdot i})^T( W^X\mathbf{x} \odot (W^H)^Th(\mathbf{y},\mathbf{x})) 
                               = (W^Y_{\cdot i} \odot
                               W^X\mathbf{x})^T
                               (W^H)^Th(\mathbf{y},\mathbf{x}) .\
\end{align*}
The derivatives of $r_i(\mathbf{y}|\mathbf{x}) - y_i$ with respect to $y_j$ are
\begin{align}
    \frac{\partial r_i(\mathbf{y}|\mathbf{x})}{\partial y_j} =&
    (W^Y_{\cdot i} \odot W^X\mathbf{x} )^T (W^H)^T \frac{\partial
      h(\mathbf{x},\mathbf{y})}{\partial y_j}
    = \frac{\partial r_j(\mathbf{y}|\mathbf{x})}{\partial y_i}
    \nonumber \label{poincare1}\\
    \frac{\partial h(\mathbf{y},\mathbf{x})}{\partial y_j} =&
    \frac{\partial h(\mathbf{u})}{\partial \mathbf{u}} W^H (W^Y_{\cdot j} \odot W^X\mathbf{x})
\end{align}
\noindent where $\mathbf{u}=   W^H((W^Y\mathbf{y}) \odot (W^X\mathbf{x}))$. 
By substituting Equation \ref{poincare1} into $\frac{\partial
  F_i}{\partial y_j}, \frac{\partial F_j}{\partial y_i}$, we have
\begin{align*}
    \frac{\partial F_i}{\partial y_j} \!  &= \! 
    \frac{\partial r_i(\mathbf{y}|\mathbf{x})}{\partial y_j} 
    \! - \! \delta_{ij}
    \! = \!
    \frac{\partial r_j(\mathbf{y}|\mathbf{x})}{\partial y_i} 
    \! - \! \delta_{ij}
    \! =  \! \frac{\partial F_j}{\partial y_i}
\end{align*}
\noindent where $\delta_{ij}=1$ for $i=j$ and $0$ for $i\neq j$.
Similarly, the derivatives of $r_i(\mathbf{y}|\mathbf{x}) - y_i$ with respect to $x_j$ are
\begin{align}
    \frac{\partial r_i(\mathbf{y}|\mathbf{x})}{\partial x_j} =&
    (W^Y_{\cdot i} \odot W^X_{\cdot j})^T(W^H)^Th(\mathbf{x},\mathbf{y}) 
    + (W^Y_{\cdot i}\odot W^X\mathbf{x})(W^H)^T \frac{\partial
      h}{\partial x_j} 
    = \frac{\partial r_j(\mathbf{y}|\mathbf{x})}{\partial
      x_i}    \label{poincare2} \nonumber ,\ \\
    \frac{\partial h(\mathbf{y},\mathbf{x})}{\partial x_j} =&
    \frac{\partial h(\mathbf{u})}{\partial \mathbf{u}} W^H (W^Y_{\cdot j} \odot W^X\mathbf{x}) .\
\end{align}
By substituting Equation \ref{poincare2} into $\frac{\partial F_i}{\partial x_j}, \frac{\partial F_j}{\partial x_i}$, this yields
\begin{align*}
    \! \frac{\partial F_i}{\partial x_j} \! &= \! 
    \frac{\partial r_i(\mathbf{x}|\mathbf{y})}{\partial x_j}
    \! = \!
    \frac{\partial r_j(\mathbf{x}|\mathbf{y})}{\partial x_i}
    \! = \! \frac{\partial F_j}{\partial x_i} .\ \nonumber 
\end{align*}

\subsection{Deriving an Energy Function} \label{App:AppendixA.1}

Integrating out the GAE's trajectory, we have 
\begin{align}
    E(\mathbf{y}|\mathbf{x}) =& \int_\mathcal{C} (r(\mathbf{y}|\mathbf{x})- \mathbf{y}) d\mathbf{y} \nonumber \\
        =& \int W^Y\left( \left(W^X\mathbf{x}) \odot W^Hh(\mathbf{u}\right)\right) d\mathbf{y} - \int  \mathbf{y} d\mathbf{y}\nonumber\\
    =& W^Y \left( \left(W^X\mathbf{x}\right) \odot  W^H\int h\left(\mathbf{u}\right) d\mathbf{u}\right) - \int \mathbf{y} d\mathbf{y}
    \label{exgypath},
\end{align}
where $\mathbf{u}$ is an auxiliary variable such that $\mathbf{u} = W^H((W^Y\mathbf{y}) \odot (W^X\mathbf{x})) $ and 
    $\frac{d\mathbf{u}}{d\mathbf{y}} = W^H(W^Y \odot (W^X\mathbf{x} \otimes \mathbf{1}_D ))$, 
where $\otimes$ is the Kronecker product. 
Consider the symmetric objective function, which is defined in Equation \ref{symobj}. Then we have to also
consider the vector field system where both symmetric cases $\mathbf{x}|\mathbf{y}$ and $\mathbf{y}|\mathbf{x}$ 
are valid. As mentioned in Section \ref{vfr}, let $\xi=[\mathbf{x};\mathbf{y}]$ and $\gamma=[\mathbf{y};\mathbf{x}]$.
As well, let $W^\xi = \diag (W^X, W^Y)$ and $W^\gamma = \diag (W^Y, W^X)$ where they are block diagonal matrices.
Consequently, the vector field becomes 
\begin{equation}
    F(\boldsymbol{\xi}|\boldsymbol{\gamma}) = r(\boldsymbol{\xi}|\boldsymbol{\gamma}) - \boldsymbol{\xi} ,\
\end{equation}
and the energy function becomes
\begin{align}
    E(\boldsymbol{\xi}|\boldsymbol{\gamma}) =& \int (r(\boldsymbol{\xi}|\boldsymbol{\gamma})- \boldsymbol{\xi}) d\boldsymbol{\xi} \nonumber\\
    =& \int (W^\xi)^T ( (W^\gamma\boldsymbol{\gamma}) \odot (W^H)^Th(\boldsymbol{u}))d\boldsymbol{\xi} - \int  \boldsymbol{\xi} d\boldsymbol{\xi}\nonumber\\
    =& (W^\xi)^T ( (W^\gamma\boldsymbol{\gamma}) \odot  (W^H)^T\int h(\boldsymbol{u}) d\boldsymbol{u}) - \int \boldsymbol{\xi} d\boldsymbol{\xi}\nonumber
\end{align}
where $\boldsymbol{u}$ is an auxiliary variable such that $\boldsymbol{u} = W^H \left((W^\xi\boldsymbol{\xi}) \odot (W^\gamma\boldsymbol{\gamma})\right)$. Then
\begin{equation}
    \frac{d \boldsymbol{u}}{d \boldsymbol{\xi}} = W^H\left(W^\xi \odot (W^\gamma\boldsymbol{\gamma} \otimes \boldsymbol{1}_D )\right) .\ \nonumber
\end{equation} 
Moreover, note that the decoder can be re-formulated as
\begin{align}
    r(\boldsymbol{\xi}|\boldsymbol{\gamma}) &= (W^\xi)^T( W^\gamma\boldsymbol{\gamma} \odot (W^H)^Th(\boldsymbol{\xi},\boldsymbol{\gamma})) \nonumber\\
                                            &= \left((W^\xi)^T \odot  (W^\gamma\boldsymbol{\gamma} \otimes \boldsymbol{1}_D)\right) (W^H)^Th(\boldsymbol{\xi},\boldsymbol{\gamma}) .\ \nonumber
\end{align}
Re-writing the first term of Equation \ref{exgypath} in terms of the auxiliary variable $\boldsymbol{u}$, the energy reduces to 
\begin{align}
    E(\boldsymbol{\xi}|\boldsymbol{\gamma}) =& \left((W^\xi)^T \odot  (W^\gamma\boldsymbol{\gamma} \otimes \boldsymbol{1}_D)\right) (W^H)^T \int h(\boldsymbol{u}) 
    \left(W^H(W^\xi \odot (W^\gamma\boldsymbol{\gamma} \otimes \boldsymbol{1}_D ))\right)^{-1}d\boldsymbol{u} - \int \boldsymbol{\xi} d\boldsymbol{\xi}\nonumber\\
    =& \left((W^\xi)^T \odot  (W^\gamma\boldsymbol{\gamma} \otimes \boldsymbol{1}_D)\right) (W^H)^T \left(
(W^\xi \odot (W^\gamma\boldsymbol{\gamma} \otimes \boldsymbol{1}_D ))W^H\right)^{-T}\int h(\boldsymbol{u}) d\boldsymbol{u}- \int \boldsymbol{\xi} d\boldsymbol{\xi}\nonumber\\
        =&\int h(\boldsymbol{u}) d\boldsymbol{u}    - \int \boldsymbol{\xi} d\boldsymbol{\xi}\nonumber\\
        = & \int h(\boldsymbol{u}) d\boldsymbol{u} -
        \frac{1}{2}\boldsymbol{\xi}^2 + \text{const} .\ \nonumber
\end{align}
\section{Relation to other types of Restricted Boltzmann Machines}\label{App:AppendixB}
\subsection{Gated Auto-encoder and Factored Gated Conditional Restricted Boltzmann Machines} \label{App:AppendixB.1}

Suppose that the hidden activation function is a sigmoid. Moreover, we define our Gated Auto-encoder
to consists of an encoder $h(\cdot)$ and decoder $r(\cdot)$ such that
\begin{align}
    h(\mathbf{x},\mathbf{y}) &= h(W^H ((W^X\mathbf{x}) \odot (W^Y\mathbf{y})) + \mathbf{b})\nonumber\\
    r(\mathbf{x}|\mathbf{y}, h) &= (W^X)^T ( (W^Y\mathbf{y}) \odot (W^H)^Th(\mathbf{x},\mathbf{y})) + \mathbf{a} ,\ \nonumber
\end{align}
where $\theta = \lbrace W^H, W^X, W^Y, \mathbf{b} \rbrace$ is the parameters of the model.
Note that the weights are not tied in this case.
The energy function for the Gated Auto-encoder will be:
\begin{align}
        E_\sigma(\mathbf{x}|\mathbf{y}) = \int & (1+\exp{(-W^H(W^X\mathbf{x})\odot(W^Y\mathbf{y})-\mathbf{b})})^{-1} d\mathbf{u}
                                                - \frac{\mathbf{x}^2}{2}+ \mathbf{ax}+ \text{const}\nonumber\\
        =  \sum_k &\log{( 1+\exp{(-W^H_{k \cdot}(W^X\mathbf{x}) \odot (W^Y\mathbf{y})-b_k)})} 
                  - \frac{\mathbf{x}^2}{2} + \mathbf{ax}+ \text{const} .\ \nonumber
\end{align}

Now consider the free energy of a Factored Gated Conditional 
Restricted Boltzmann Machine (FCRBM).

The energy function of a FCRBM with Gaussian visible units and Bernoulli hidden units is defined by
\begin{equation}
    E(\mathbf{x},\mathbf{h}|\mathbf{y}) = \frac{(\mathbf{a}-\mathbf{x})^2}{2\mathbf{\sigma}^2}
    - \mathbf{bh} - \sum_f W^X_{f\cdot}\mathbf{x} \odot W^Y_{f\cdot}\mathbf{y} \odot W^H_{f\cdot}\mathbf{h} .\ \nonumber
\end{equation}
Given $\mathbf{y}$, the conditional probability density assigned by the FCRBM to data point $\mathbf{x}$ is
\begin{align}
    p(\mathbf{x}|\mathbf{y}) & = \frac{\sum_\mathbf{h}\exp{-(E(\mathbf{x},\mathbf{h}|\mathbf{y}))}}{Z(\mathbf{y})} = \frac{\exp{(-F(\mathbf{x}|\mathbf{y}))}}{Z(\mathbf{y})}\nonumber\\
    - F(\mathbf{x}|\mathbf{y}) & = \log{\left(\sum_\mathbf{h} \exp\left(-E(\mathbf{x},\mathbf{h}|\mathbf{y})\right)\right)}\nonumber
\end{align}
where $Z(\mathbf{y}) = \sum_\mathbf{x,h} \exp \left(E(\mathbf{x},\mathbf{h}|\mathbf{y})\right)$ is the partition function
and $F(\mathbf{x}|\mathbf{y})$ is the free energy function. Expanding the free energy
function, we get
\begin{align}
    - F(\mathbf{x}|\mathbf{y}) =& \log{\sum_\mathbf{h} \exp\left(-E(\mathbf{x},\mathbf{h}|\mathbf{y})\right)}\nonumber\\
                             =& \log\sum_\mathbf{h} \exp\left(\frac{-(\mathbf{a}-\mathbf{x})^2}{2\mathbf{\sigma}^2} + \mathbf{bh} 
                               + \sum_f W^X_{f\cdot} \mathbf{x}\odot W^Y_{f\cdot}\mathbf{y} \odot W^H_{f\cdot}\mathbf{h}\right)\nonumber\\
                               =& - \frac{(\mathbf{a}-\mathbf{x})^2}{2\mathbf{\sigma}^2} + \log\left(\sum_\mathbf{h} \exp \left(\mathbf{bh} 
                           + \sum_f W^X_{f\cdot}\mathbf{x} \odot W^Y_{f\cdot}\mathbf{y} \odot W^H_{f\cdot}\mathbf{h}\right)\right)\nonumber\\
                                =&  - \frac{(\mathbf{a}-\mathbf{x})^2}{2\mathbf{\sigma}^2} 
                                    + \log\left(\sum_\mathbf{h} \prod_k \exp \left(b_kh_k 
                                    + \sum_f (W^X_{f\cdot} \mathbf{x}\odot W^Y_{f\cdot}\mathbf{y}) \odot W^H_{fk}h_k \right)\right)\nonumber\\
                                =& - \frac{(\mathbf{a}-\mathbf{x})^2}{2\mathbf{\sigma}^2} + \sum_k \log \left( 1 + \exp\left( b_k + 
                                \sum_f \left((W^H_{fk})^T(W^X\mathbf{x} \odot W^Y\mathbf{y})\right)\right)\right) .\ \nonumber
\end{align}
    Note that we can center the data by subtracting mean of $\mathbf{x}$ and dividing by its standard deviation, and therefore assume that $\sigma^2=1$. Substituting, we have
\begin{align}
    - F(\mathbf{x}|\mathbf{y})=& -\frac{(\mathbf{a}-\mathbf{x})^2}{2} +  \sum_k \log \left(1 +
\exp\left(- b_k - \sum_f(W^H_{fk})^T(W^X\mathbf{x} \odot W^Y\mathbf{y}) \right)\right)\nonumber\\
                               =&  \sum_k \log \left( 1 + \exp\left( b_k 
                + \sum_f(W^H_{fk})^T(W^X\mathbf{x} \odot W^Y\mathbf{y}) \right)\right)
                 - \mathbf{a}^2+\mathbf{ax}-\frac{\mathbf{x}^2}{2}\nonumber\\
                 =& \sum_k \log \left( 1 + \exp\left( b_k 
             + \sum_f(W^H_{fk })^T(W^X\mathbf{x} \odot W^Y\mathbf{y})\right)\right)
                  +\mathbf{ax}-\frac{\mathbf{x}^2}{2}+\text{const}\nonumber
\end{align}
Letting $W^H = (W^H)^T$, we get 
\begin{equation}
    = \sum_k \log \left( 1 + \exp \left( b_k + \sum_fW^H_{k f }(W^X\mathbf{x} \odot W^Y\mathbf{y})\right)\right)
                  +\mathbf{ax}-\frac{\mathbf{x}^2}{2}+\text{const}\nonumber
\end{equation}
Hence, the Conditional Gated Auto-encoder and the FCRBM are equal up to a constant.

\subsection{Gated Auto-encoder and mean-covariance Restricted Boltzmann Machines} \label{App:AppendixB.2}
\begin{theorem}
    Consider a covariance auto-encoder with an encoder and decoder, 
    \begin{align*}
        h(\mathbf{x},\mathbf{x}) &= h(W^H((W^F\mathbf{x})^2) + \mathbf{b})\nonumber\\
        r(\mathbf{x}|\mathbf{y}=\mathbf{x}, h) &= (W^F)^T(W^F\mathbf{y} \odot (W^H)^Th(\mathbf{x},\mathbf{y}))  + \mathbf{a},\nonumber
    \end{align*}
    where $\theta = \lbrace W^F, W^H, \mathbf{a}, \mathbf{b}\rbrace$ are the parameters of the model. 
    Moreover, consider a covariance Restricted Boltzmann Machine with Gaussian distribution
    over the visibles and Bernoulli distribution over the hiddens, such that its energy function
    is defined by 
    \begin{equation*}
        E^c(\mathbf{x},\mathbf{h}) = \frac{(\mathbf{a}-\mathbf{x})^2}{\sigma^2}- \sum_f P\mathbf{h} (C\mathbf{x})^2 -\mathbf{bh} ,\
    \end{equation*} 
    \noindent where $\theta = \lbrace P, C, \mathbf{a},
    \mathbf{b}\rbrace$ are its parameters.
    Then the energy function for a covariance Auto-encoder with dynamics
    $r(\mathbf{x}|\mathbf{y}) - \mathbf{x}$ is equivalent to the free
    energy of a covariance Restricted Boltzmann Machine. 
    The energy function of the covariance Auto-encoder is 
    \begin{equation}
        E(\mathbf{x},\mathbf{x}) = \sum_k \log (1+\exp(W^H(W^F\mathbf{x})^2+\mathbf{b})) - \mathbf{x}^2+\text{const}
    \end{equation}
\end{theorem}
\begin{proof}
    Note that the covariance auto-encoder is the same as a regular Gated Auto-encoder, 
    but setting $\mathbf{y} = \mathbf{x}$ and making the factor loading matrices the same, i.e.~$W^F=W^Y=W^X$.
    Then applying the general energy equation for GAE, Equation
    \ref{exgy2}, to the covariance auto-encoder,
    we get
    \begin{align}
        E(\mathbf{x},\mathbf{x}) =& \int h(\mathbf{u}) d\mathbf{u} - \frac{1}{2}\mathbf{x}^2 + \text{const}\nonumber\\
            =& \sum_k \log (1+\exp(W^H(W^F\mathbf{x})^2+\mathbf{b})) 
            - \mathbf{x}^2+\mathbf{ax}+\text{const} ,\
    \end{align}
    where $\mathbf{u} = W^H(W^F\mathbf{x})^2+\mathbf{b}$.

    Now consider the free energy of the mean-covariance 
    Restricted Boltzmann Machine (mcRBM) with Gaussian distribution
    over the visible units and Bernoulli distribution over the hidden units:
    \begin{align}
        - F(\mathbf{x}|\mathbf{y}) =& \log{\sum_\mathbf{h} \exp\left(-E(\mathbf{x},\mathbf{h}|\mathbf{y})\right)}\nonumber\\
        =& \log\sum_h \exp \left(-\frac{(\mathbf{a}-\mathbf{x})^2}{\sigma^2}+ (P\mathbf{h})(C\mathbf{x})^2 +\mathbf{bh}\right)\nonumber\\
        =& \log\sum_h \prod_k\exp \left(-\frac{(\mathbf{a}-\mathbf{x})^2}{\sigma^2}+ \sum_f(P_{fk}h_k)(C\mathbf{x})^2 + b_k h_k\right)\nonumber\\
        =& \sum_k \log \left(1+\exp\left(\sum_f(P_{fk}h_k)(C\mathbf{x})^2\right)\right) - \frac{(\mathbf{a}-\mathbf{x})^2}{\sigma^2} .\ \nonumber
    \end{align}
    As before, we can centre the data by subtracting mean of $\mathbf{x}$ and dividing by its standard deviation, and therefore assume that $\sigma^2=1$. Substituting, we have
    \begin{equation}
        =  \sum_k \log \left(1+\exp\left(\sum_f(P_{fk}h_k)(C\mathbf{x})^2\right)\right)-(\mathbf{a}-\mathbf{x})^2 .\
    \end{equation}
    Letting $W^H = P^T$ and $W^F=C$, we get 
    \begin{equation}
        =  \sum_k \log \left(1+\exp\left(\sum_f(P_{fk}h_k)(C\mathbf{x})^2\right)\right) - \mathbf{x}^2 + \mathbf{ax} + \text{const} .\
    \end{equation}
    Therefore, the two equations are equivalent.
\end{proof}
\fi

\appendix
\section{Gated Auto-encoder Scoring} \label{App:AppendixA}

\subsection{Vector field representation} \label{App:poincare}
To check that the vector
field can be written as the derivative of a scalar field,
we can submit to Poincar\'e's
integrability criterion: For some open, simple connected set
$\mathcal{U}$, a continuously differentiable function $F:\mathcal{U}
\rightarrow \Re^m$ defines a gradient field if and only if
\begin{align}
    \frac{\partial F_i(\mathbf{y})}{\partial y_j} = 
    \frac{\partial F_j(\mathbf{y})}{\partial y_i}, \text{ } \forall i,j=1 \cdots n.\nonumber
\end{align}
Considering the GAE, note that $i^{th}$ component of the decoder
$r_i(\mathbf{y}|\mathbf{x})$ can be rewritten as
\begin{align*}
    r_i(\mathbf{y}|\mathbf{x}) &= (W^Y_{\cdot i})^T( W^X\mathbf{x} \odot (W^H)^Th(\mathbf{y},\mathbf{x})) 
                               = (W^Y_{\cdot i} \odot
                               W^X\mathbf{x})^T
                               (W^H)^Th(\mathbf{y},\mathbf{x}) .\
\end{align*}
The derivatives of $r_i(\mathbf{y}|\mathbf{x}) - y_i$ with respect to $y_j$ are
\begin{align}
    \frac{\partial r_i(\mathbf{y}|\mathbf{x})}{\partial y_j} =&
    (W^Y_{\cdot i} \odot W^X\mathbf{x} )^T (W^H)^T \frac{\partial
      h(\mathbf{x},\mathbf{y})}{\partial y_j}
    = \frac{\partial r_j(\mathbf{y}|\mathbf{x})}{\partial y_i}
    \nonumber \label{poincare1}\\
    \frac{\partial h(\mathbf{y},\mathbf{x})}{\partial y_j} =&
    \frac{\partial h(\mathbf{u})}{\partial \mathbf{u}} W^H (W^Y_{\cdot j} \odot W^X\mathbf{x})
\end{align}
\noindent where $\mathbf{u}=   W^H((W^Y\mathbf{y}) \odot (W^X\mathbf{x}))$. 
By substituting Equation \ref{poincare1} into $\frac{\partial
  F_i}{\partial y_j}, \frac{\partial F_j}{\partial y_i}$, we have
\begin{align*}
    \frac{\partial F_i}{\partial y_j} \!  &= \! 
    \frac{\partial r_i(\mathbf{y}|\mathbf{x})}{\partial y_j} 
    \! - \! \delta_{ij}
    \! = \!
    \frac{\partial r_j(\mathbf{y}|\mathbf{x})}{\partial y_i} 
    \! - \! \delta_{ij}
    \! =  \! \frac{\partial F_j}{\partial y_i}
\end{align*}
\noindent where $\delta_{ij}=1$ for $i=j$ and $0$ for $i\neq j$.
Similarly, the derivatives of $r_i(\mathbf{y}|\mathbf{x}) - y_i$ with respect to $x_j$ are
\begin{align}
    \frac{\partial r_i(\mathbf{y}|\mathbf{x})}{\partial x_j} =&
    (W^Y_{\cdot i} \odot W^X_{\cdot j})^T(W^H)^Th(\mathbf{x},\mathbf{y}) 
    + (W^Y_{\cdot i}\odot W^X\mathbf{x})(W^H)^T \frac{\partial
      h}{\partial x_j} 
    = \frac{\partial r_j(\mathbf{y}|\mathbf{x})}{\partial
      x_i}    \label{poincare2} \nonumber ,\ \\
    \frac{\partial h(\mathbf{y},\mathbf{x})}{\partial x_j} =&
    \frac{\partial h(\mathbf{u})}{\partial \mathbf{u}} W^H (W^Y_{\cdot j} \odot W^X\mathbf{x}) .\
\end{align}
By substituting Equation \ref{poincare2} into $\frac{\partial F_i}{\partial x_j}, \frac{\partial F_j}{\partial x_i}$, this yields
\begin{align*}
    \! \frac{\partial F_i}{\partial x_j} \! &= \! 
    \frac{\partial r_i(\mathbf{x}|\mathbf{y})}{\partial x_j}
    \! = \!
    \frac{\partial r_j(\mathbf{x}|\mathbf{y})}{\partial x_i}
    \! = \! \frac{\partial F_j}{\partial x_i} .\ \nonumber 
\end{align*}

\subsection{Deriving an Energy Function} \label{App:AppendixA.1}

Integrating out the GAE's trajectory, we have 
\begin{align}
    E(\mathbf{y}|\mathbf{x}) =& \int_\mathcal{C} (r(\mathbf{y}|\mathbf{x})- \mathbf{y}) d\mathbf{y} \nonumber \\
        =& \int W^Y\left( \left(W^X\mathbf{x}) \odot W^Hh(\mathbf{u}\right)\right) d\mathbf{y} - \int  \mathbf{y} d\mathbf{y}\nonumber\\
    =& W^Y \left( \left(W^X\mathbf{x}\right) \odot  W^H\int h\left(\mathbf{u}\right) d\mathbf{u}\right) - \int \mathbf{y} d\mathbf{y}
    \label{exgypath},
\end{align}
where $\mathbf{u}$ is an auxiliary variable such that $\mathbf{u} = W^H((W^Y\mathbf{y}) \odot (W^X\mathbf{x})) $ and 
    $\frac{d\mathbf{u}}{d\mathbf{y}} = W^H(W^Y \odot (W^X\mathbf{x} \otimes \mathbf{1}_D ))$, 
where $\otimes$ is the Kronecker product. 
Consider the symmetric objective function, which is defined in Equation \ref{symobj}. Then we have to also
consider the vector field system where both symmetric cases $\mathbf{x}|\mathbf{y}$ and $\mathbf{y}|\mathbf{x}$ 
are valid. As mentioned in Section 3.1, let $\xi=[\mathbf{x};\mathbf{y}]$ and $\gamma=[\mathbf{y};\mathbf{x}]$.
As well, let $W^\xi = \diag (W^X, W^Y)$ and $W^\gamma = \diag (W^Y, W^X)$ where they are block diagonal matrices.
Consequently, the vector field becomes 
\begin{equation}
    F(\boldsymbol{\xi}|\boldsymbol{\gamma}) = r(\boldsymbol{\xi}|\boldsymbol{\gamma}) - \boldsymbol{\xi} ,\
\end{equation}
and the energy function becomes
\begin{align}
    E(\boldsymbol{\xi}|\boldsymbol{\gamma}) =& \int (r(\boldsymbol{\xi}|\boldsymbol{\gamma})- \boldsymbol{\xi}) d\boldsymbol{\xi} \nonumber\\
    =& \int (W^\xi)^T ( (W^\gamma\boldsymbol{\gamma}) \odot (W^H)^Th(\boldsymbol{u}))d\boldsymbol{\xi} - \int  \boldsymbol{\xi} d\boldsymbol{\xi}\nonumber\\
    =& (W^\xi)^T ( (W^\gamma\boldsymbol{\gamma}) \odot  (W^H)^T\int h(\boldsymbol{u}) d\boldsymbol{u}) - \int \boldsymbol{\xi} d\boldsymbol{\xi}\nonumber
\end{align}
where $\boldsymbol{u}$ is an auxiliary variable such that $\boldsymbol{u} = W^H \left((W^\xi\boldsymbol{\xi}) \odot (W^\gamma\boldsymbol{\gamma})\right)$. Then
\begin{equation}
    \frac{d \boldsymbol{u}}{d \boldsymbol{\xi}} = W^H\left(W^\xi \odot (W^\gamma\boldsymbol{\gamma} \otimes \boldsymbol{1}_D )\right) .\ \nonumber
\end{equation} 
Moreover, note that the decoder can be re-formulated as
\begin{align}
    r(\boldsymbol{\xi}|\boldsymbol{\gamma}) &= (W^\xi)^T( W^\gamma\boldsymbol{\gamma} \odot (W^H)^Th(\boldsymbol{\xi},\boldsymbol{\gamma})) \nonumber\\
                                            &= \left((W^\xi)^T \odot  (W^\gamma\boldsymbol{\gamma} \otimes \boldsymbol{1}_D)\right) (W^H)^Th(\boldsymbol{\xi},\boldsymbol{\gamma}) .\ \nonumber
\end{align}
Re-writing the first term of Equation \ref{exgypath} in terms of the auxiliary variable $\boldsymbol{u}$, the energy reduces to 
\begin{align}
    E(\boldsymbol{\xi}|\boldsymbol{\gamma}) =& \left((W^\xi)^T \odot  (W^\gamma\boldsymbol{\gamma} \otimes \boldsymbol{1}_D)\right) (W^H)^T \int h(\boldsymbol{u}) 
    \left(W^H(W^\xi \odot (W^\gamma\boldsymbol{\gamma} \otimes \boldsymbol{1}_D ))\right)^{-1}d\boldsymbol{u} - \int \boldsymbol{\xi} d\boldsymbol{\xi}\nonumber\\
    =& \left((W^\xi)^T \odot  (W^\gamma\boldsymbol{\gamma} \otimes \boldsymbol{1}_D)\right) (W^H)^T \left(
(W^\xi \odot (W^\gamma\boldsymbol{\gamma} \otimes \boldsymbol{1}_D ))W^H\right)^{-T}\int h(\boldsymbol{u}) d\boldsymbol{u}- \int \boldsymbol{\xi} d\boldsymbol{\xi}\nonumber\\
        =&\int h(\boldsymbol{u}) d\boldsymbol{u}    - \int \boldsymbol{\xi} d\boldsymbol{\xi}\nonumber\\
        = & \int h(\boldsymbol{u}) d\boldsymbol{u} -
        \frac{1}{2}\boldsymbol{\xi}^2 + \text{const} .\ \nonumber
\end{align}
\section{Relation to other types of Restricted Boltzmann Machines}\label{App:AppendixB}
\subsection{Gated Auto-encoder and Factored Gated Conditional Restricted Boltzmann Machines} \label{App:AppendixB.1}

Suppose that the hidden activation function is a sigmoid. Moreover, we define our Gated Auto-encoder
to consists of an encoder $h(\cdot)$ and decoder $r(\cdot)$ such that
\begin{align}
    h(\mathbf{x},\mathbf{y}) &= h(W^H ((W^X\mathbf{x}) \odot (W^Y\mathbf{y})) + \mathbf{b})\nonumber\\
    r(\mathbf{x}|\mathbf{y}, h) &= (W^X)^T ( (W^Y\mathbf{y}) \odot (W^H)^Th(\mathbf{x},\mathbf{y})) + \mathbf{a} ,\ \nonumber
\end{align}
where $\theta = \lbrace W^H, W^X, W^Y, \mathbf{b} \rbrace$ is the parameters of the model.
Note that the weights are not tied in this case.
The energy function for the Gated Auto-encoder will be:
\begin{align}
        E_\sigma(\mathbf{x}|\mathbf{y}) = \int & (1+\exp{(-W^H(W^X\mathbf{x})\odot(W^Y\mathbf{y})-\mathbf{b})})^{-1} d\mathbf{u}
                                                - \frac{\mathbf{x}^2}{2}+ \mathbf{ax}+ \text{const}\nonumber\\
        =  \sum_k &\log{( 1+\exp{(-W^H_{k \cdot}(W^X\mathbf{x}) \odot (W^Y\mathbf{y})-b_k)})} 
                  - \frac{\mathbf{x}^2}{2} + \mathbf{ax}+ \text{const} .\ \nonumber
\end{align}

Now consider the free energy of a Factored Gated Conditional 
Restricted Boltzmann Machine (FCRBM).

The energy function of a FCRBM with Gaussian visible units and Bernoulli hidden units is defined by
\begin{equation}
    E(\mathbf{x},\mathbf{h}|\mathbf{y}) = \frac{(\mathbf{a}-\mathbf{x})^2}{2\mathbf{\sigma}^2}
    - \mathbf{bh} - \sum_f W^X_{f\cdot}\mathbf{x} \odot W^Y_{f\cdot}\mathbf{y} \odot W^H_{f\cdot}\mathbf{h} .\ \nonumber
\end{equation}
Given $\mathbf{y}$, the conditional probability density assigned by the FCRBM to data point $\mathbf{x}$ is
\begin{align}
    p(\mathbf{x}|\mathbf{y}) & = \frac{\sum_\mathbf{h}\exp{-(E(\mathbf{x},\mathbf{h}|\mathbf{y}))}}{Z(\mathbf{y})} = \frac{\exp{(-F(\mathbf{x}|\mathbf{y}))}}{Z(\mathbf{y})}\nonumber\\
    - F(\mathbf{x}|\mathbf{y}) & = \log{\left(\sum_\mathbf{h} \exp\left(-E(\mathbf{x},\mathbf{h}|\mathbf{y})\right)\right)}\nonumber
\end{align}
where $Z(\mathbf{y}) = \sum_\mathbf{x,h} \exp \left(E(\mathbf{x},\mathbf{h}|\mathbf{y})\right)$ is the partition function
and $F(\mathbf{x}|\mathbf{y})$ is the free energy function. Expanding the free energy
function, we get
\begin{align}
    - F(\mathbf{x}|\mathbf{y}) =& \log{\sum_\mathbf{h} \exp\left(-E(\mathbf{x},\mathbf{h}|\mathbf{y})\right)}\nonumber\\
                             =& \log\sum_\mathbf{h} \exp\left(\frac{-(\mathbf{a}-\mathbf{x})^2}{2\mathbf{\sigma}^2} + \mathbf{bh} 
                               + \sum_f W^X_{f\cdot} \mathbf{x}\odot W^Y_{f\cdot}\mathbf{y} \odot W^H_{f\cdot}\mathbf{h}\right)\nonumber\\
                               =& - \frac{(\mathbf{a}-\mathbf{x})^2}{2\mathbf{\sigma}^2} + \log\left(\sum_\mathbf{h} \exp \left(\mathbf{bh} 
                           + \sum_f W^X_{f\cdot}\mathbf{x} \odot W^Y_{f\cdot}\mathbf{y} \odot W^H_{f\cdot}\mathbf{h}\right)\right)\nonumber\\
                                =&  - \frac{(\mathbf{a}-\mathbf{x})^2}{2\mathbf{\sigma}^2} 
                                    + \log\left(\sum_\mathbf{h} \prod_k \exp \left(b_kh_k 
                                    + \sum_f (W^X_{f\cdot} \mathbf{x}\odot W^Y_{f\cdot}\mathbf{y}) \odot W^H_{fk}h_k \right)\right)\nonumber\\
                                =& - \frac{(\mathbf{a}-\mathbf{x})^2}{2\mathbf{\sigma}^2} + \sum_k \log \left( 1 + \exp\left( b_k + 
                                \sum_f \left((W^H_{fk})^T(W^X\mathbf{x} \odot W^Y\mathbf{y})\right)\right)\right) .\ \nonumber
\end{align}
    Note that we can center the data by subtracting mean of $\mathbf{x}$ and dividing by its standard deviation, and therefore assume that $\sigma^2=1$. Substituting, we have
\begin{align}
    - F(\mathbf{x}|\mathbf{y})=& -\frac{(\mathbf{a}-\mathbf{x})^2}{2} +  \sum_k \log \left(1 +
\exp\left(- b_k - \sum_f(W^H_{fk})^T(W^X\mathbf{x} \odot W^Y\mathbf{y}) \right)\right)\nonumber\\
                               =&  \sum_k \log \left( 1 + \exp\left( b_k 
                + \sum_f(W^H_{fk})^T(W^X\mathbf{x} \odot W^Y\mathbf{y}) \right)\right)
                 - \mathbf{a}^2+\mathbf{ax}-\frac{\mathbf{x}^2}{2}\nonumber\\
                 =& \sum_k \log \left( 1 + \exp\left( b_k 
             + \sum_f(W^H_{fk })^T(W^X\mathbf{x} \odot W^Y\mathbf{y})\right)\right)
                  +\mathbf{ax}-\frac{\mathbf{x}^2}{2}+\text{const}\nonumber
\end{align}
Letting $W^H = (W^H)^T$, we get 
\begin{equation}
    = \sum_k \log \left( 1 + \exp \left( b_k + \sum_fW^H_{k f }(W^X\mathbf{x} \odot W^Y\mathbf{y})\right)\right)
                  +\mathbf{ax}-\frac{\mathbf{x}^2}{2}+\text{const}\nonumber
\end{equation}
Hence, the Conditional Gated Auto-encoder and the FCRBM are equal up to a constant.

\subsection{Gated Auto-encoder and mean-covariance Restricted Boltzmann Machines} \label{App:AppendixB.2}
\begin{theorem}
    Consider a covariance auto-encoder with an encoder and decoder, 
    \begin{align*}
        h(\mathbf{x},\mathbf{x}) &= h(W^H((W^F\mathbf{x})^2) + \mathbf{b})\nonumber\\
        r(\mathbf{x}|\mathbf{y}=\mathbf{x}, h) &= (W^F)^T(W^F\mathbf{y} \odot (W^H)^Th(\mathbf{x},\mathbf{y}))  + \mathbf{a},\nonumber
    \end{align*}
    where $\theta = \lbrace W^F, W^H, \mathbf{a}, \mathbf{b}\rbrace$ are the parameters of the model. 
    Moreover, consider a covariance Restricted Boltzmann Machine with Gaussian distribution
    over the visibles and Bernoulli distribution over the hiddens, such that its energy function
    is defined by 
    \begin{equation*}
        E^c(\mathbf{x},\mathbf{h}) = \frac{(\mathbf{a}-\mathbf{x})^2}{\sigma^2}- \sum_f P\mathbf{h} (C\mathbf{x})^2 -\mathbf{bh} ,\
    \end{equation*} 
    \noindent where $\theta = \lbrace P, C, \mathbf{a},
    \mathbf{b}\rbrace$ are its parameters.
    Then the energy function for a covariance Auto-encoder with dynamics
    $r(\mathbf{x}|\mathbf{y}) - \mathbf{x}$ is equivalent to the free
    energy of a covariance Restricted Boltzmann Machine. 
    The energy function of the covariance Auto-encoder is 
    \begin{equation}
        E(\mathbf{x},\mathbf{x}) = \sum_k \log (1+\exp(W^H(W^F\mathbf{x})^2+\mathbf{b})) - \mathbf{x}^2+\text{const}
    \end{equation}
\end{theorem}
\begin{proof}
    Note that the covariance auto-encoder is the same as a regular Gated Auto-encoder, 
    but setting $\mathbf{y} = \mathbf{x}$ and making the factor loading matrices the same, i.e.~$W^F=W^Y=W^X$.
    Then applying the general energy equation for GAE, Equation
    \ref{exgy2}, to the covariance auto-encoder,
    we get
    \begin{align}
        E(\mathbf{x},\mathbf{x}) =& \int h(\mathbf{u}) d\mathbf{u} - \frac{1}{2}\mathbf{x}^2 + \text{const}\nonumber\\
            =& \sum_k \log (1+\exp(W^H(W^F\mathbf{x})^2+\mathbf{b})) 
            - \mathbf{x}^2+\mathbf{ax}+\text{const} ,\
    \end{align}
    where $\mathbf{u} = W^H(W^F\mathbf{x})^2+\mathbf{b}$.

    Now consider the free energy of the mean-covariance 
    Restricted Boltzmann Machine (mcRBM) with Gaussian distribution
    over the visible units and Bernoulli distribution over the hidden units:
    \begin{align}
        - F(\mathbf{x}|\mathbf{y}) =& \log{\sum_\mathbf{h} \exp\left(-E(\mathbf{x},\mathbf{h}|\mathbf{y})\right)}\nonumber\\
        =& \log\sum_h \exp \left(-\frac{(\mathbf{a}-\mathbf{x})^2}{\sigma^2}+ (P\mathbf{h})(C\mathbf{x})^2 +\mathbf{bh}\right)\nonumber\\
        =& \log\sum_h \prod_k\exp \left(-\frac{(\mathbf{a}-\mathbf{x})^2}{\sigma^2}+ \sum_f(P_{fk}h_k)(C\mathbf{x})^2 + b_k h_k\right)\nonumber\\
        =& \sum_k \log \left(1+\exp\left(\sum_f(P_{fk}h_k)(C\mathbf{x})^2\right)\right) - \frac{(\mathbf{a}-\mathbf{x})^2}{\sigma^2} .\ \nonumber
    \end{align}
    As before, we can center the data by subtracting mean of $\mathbf{x}$ and dividing by its standard deviation, and therefore assume that $\sigma^2=1$. Substituting, we have
    \begin{equation}
        =  \sum_k \log \left(1+\exp\left(\sum_f(P_{fk}h_k)(C\mathbf{x})^2\right)\right)-(\mathbf{a}-\mathbf{x})^2 .\
    \end{equation}
    Letting $W^H = P^T$ and $W^F=C$, we get 
    \begin{equation}
        =  \sum_k \log \left(1+\exp\left(\sum_f(P_{fk}h_k)(C\mathbf{x})^2\right)\right) - \mathbf{x}^2 + \mathbf{ax} + \text{const} .\
    \end{equation}
    Therefore, the two equations are equivalent.
\end{proof}

\end{document}